\def\year{2020}\relax
\documentclass[letterpaper]{article} 
\usepackage{aaai20}  
\usepackage{times}  
\usepackage{helvet} 
\usepackage{courier}  
\usepackage[hyphens]{url}  
\usepackage{graphicx} 
\usepackage{graphicx}  
\usepackage{booktabs}   
\usepackage{amsfonts}
\usepackage{nicefrac}    
\usepackage{microtype}    
\usepackage[utf8]{inputenc}
\usepackage{amsmath,amssymb}
\usepackage{mathtools}
\usepackage{amsthm}  
\usepackage{commath}
\usepackage{bm}
\usepackage{algorithm}
\usepackage{algpseudocode}
\usepackage{varwidth}
\usepackage[export]{adjustbox}
\usepackage{subfig}
\usepackage{nohyperref}
\usepackage{url}

\usepackage{etoolbox}\AtBeginEnvironment{algorithmic}{\small}

\usepackage{thm-restate}

\declaretheorem[name=Corollary,numberwithin=section]{coroll}

\usepackage[textsize=tiny]{todonotes}
\usepackage{xspace}

\DeclareRobustCommand{\ie}{i.e.,\@\xspace}
\DeclareRobustCommand{\wrt}{w.r.t.\@\xspace}
\usepackage{amssymb}

\newcommand{\EV}{\mathbb{E}}

\newcommand{\Norm}[2][\infty]{\left\|#2\right\|_{#1}}

\DeclareMathOperator*{\argmax}{arg\,max}

\title{An Intrinsically-Motivated Approach for Learning \\ Highly Exploring and Fast Mixing Policies}
\author{%
  Mirco Mutti \\
  Politecnico di Milano, Milan, Italy \\
  Universit\`a di Bologna, Bologna, Italy \\
  \texttt{mirco.mutti@polimi.it} \\
   \And
   Marcello Restelli \\
   Politecnico di Milano, Milan, Italy \\
   \texttt{marcello.restelli@polimi.it} \\
}

\begin{document}

\maketitle

\begin{abstract}
What is a good exploration strategy for an agent that interacts with an environment in the absence of external rewards? Ideally, we would like to get a policy driving towards a uniform state-action visitation (highly exploring) in a minimum number of steps (fast mixing), in order to ease efficient learning of any goal-conditioned policy later on. Unfortunately, it is remarkably arduous to directly learn an optimal policy of this nature. In this paper, we propose a novel surrogate objective for learning highly exploring and fast mixing policies, which focuses on maximizing a lower bound to the entropy of the steady-state distribution induced by the policy. In particular, we introduce three novel lower bounds, that lead to as many optimization problems, that tradeoff the theoretical guarantees with computational complexity. Then, we present a model-based reinforcement learning algorithm, IDE$^{3}$AL, to learn an optimal policy according to the introduced objective. Finally, we provide an empirical evaluation of this algorithm on a set of hard-exploration tasks.
\end{abstract}

\section{Introduction}

In general, the Reinforcement Learning (RL) framework~\cite{sutton2018reinforcement} assumes the presence of a reward signal coming from a, potentially unknown, environment to a learning agent. When this signal is sufficiently informative about the utility of the agent's decisions, RL has proved to be rather successful in solving challenging tasks, even at a super-human level~\cite{mnih2015human,silver2017mastering}. However, in most real-world scenarios, we cannot rely on a well-shaped, complete reward signal. This may prevent the agent from learning anything until, while performing random actions, it eventually stumbles into some sort of external reward. Thus, what is a good objective for a learning agent to pursue, in the absence of an external reward signal, to prepare itself to learn efficiently, eventually, a goal-conditioned policy?

Intrinsic motivation~\cite{chentanez2005intrinsically,oudeyer2009topology} traditionally tries to answer this pressing question by designing self-motivated goals that favor exploration. In a curiosity-driven approach, first proposed in~\cite{schmidhuber1991possibility}, the intrinsic objective encourages the agent to explore novel states by rewarding prediction errors \cite{stadie2015incentivizing,pathak2017curiosity,burda2018large,burda2018exploration}.
On a similar flavor, other works propose to relate an intrinsic reward to some sort of learning progress~\cite{lopes2012exploration} or information gain~\cite{mohamed2015variational,houthooft2016vime}, stimulating the agent's empowerment over the environment. 
Count-based approaches~\cite{bellemare2016unifying,tang2017exploration,ostrovski2017count} consider exploration bonuses proportional to the state visitation frequencies, assigning high rewards to rarely visited states.
Athough the mentioned approaches have been relatively effective in solving sparse-rewards, hard-exploration tasks~\cite{pathak2017curiosity,burda2018exploration}, they have some common limitations that may affect their ability to methodically explore an environment in the absence of external rewards, as pointed out in~\cite{ecoffet2019go}. Especially, due to the consumable nature of their intrinsic bonuses, the learning agent could prematurely lose interest in a frontier of high rewards (\emph{detachment}). Furthermore, the agent may suffer from \emph{derailment} by trying to return to a promising state, previously discovered, if a na\"ive exploratory mechanism, such as $\epsilon$-greedy, is combined to the intrinsic motivation mechanism (which is often the case). 
To overcome these limitations, recent works suggest alternative approaches to motivate the agent towards a more systematic exploration of the environment~\cite{hazan2018provably,ecoffet2019go}. Especially, in~\cite{hazan2018provably} the authors consider an intrinsic objective which is directed to the maximization of an entropic measure over the state distribution induced by a policy. Then, they provide a provably efficient algorithm to learn a mixture of deterministic policies that is overall optimal \wrt the maximum-entropy exploration objective.
To the best of our knowledge, none of the mentioned approaches explicitly address the related aspect of the mixing time of an exploratory policy, which represents the time it takes for the policy to reach its full capacity in terms of exploration. Nonetheless, in many cases we would like to maximize the probability of reaching any potential target state having a fairly limited number of interactions at hand for exploring the environment. Notably, this context presents some analogies to the problem of maximizing the efficiency of a random walk~\cite{hassibi2014optimized}.

In this paper, we present a novel approach to learn exploratory policies that are, at the same time, highly exploring and fast mixing. 
In Section~\ref{sec:objective}, we propose a surrogate objective to address the problem of maximum-entropy exploration over both the state space (Section~\ref{sec:highly_exploring_state}) and the action space (Section~\ref{sec:highly_exploring_action}). The idea is to search for a policy that maximizes a lower bound to the entropy of the induced steady-state distribution. We introduce three new lower bounds and the corresponding optimization problems, discussing their pros and cons.
Furthermore, we discuss how to complement the introduced objective to account for the mixing time of the learned policy (Section~\ref{sec:fast_mixing}).
In Section~\ref{sec:algorithm}, we present the \emph{Intrinsically-Driven Effective and Efficient Exploration ALgorithm} (IDE$^3$AL), a novel, model-based, reinforcement learning method to learn highly exploring and fast mixing policies through iterative optimizations of the introduced objective.
In Section~\ref{sec:experiments}, we provide an empirical evaluation to illustrate the merits of our approach on hard-exploration, finite domains, and to show how it fares in comparison to count-based and maximum-entropy approaches.
Finally, in Section~\ref{sec:discussion}, we discuss the proposed approach and related works. The proofs of the Theorems are reported in Appendix~\ref{apx:proof}\footnote{A complete version of the paper, which includes the Appendix, is available at https://arxiv.org/abs/1907.04662}.

\section{Preliminaries}
A discrete-time Markov Decision Process (MDP)~\cite{puterman2014markov} is defined as a tuple $\mathcal{M} = (\mathcal{S}, \mathcal{A}, P, R, d_0 )$, where $\mathcal{S}$ is the state space, $\mathcal{A}$ is the action space, $P(s' | s,a)$ is a Markovian transition model defining the distribution of the next state $s'$ given the current state $s$ and action $a$, $R$ is the reward function, such that $R(s,a)$ is the expected immediate reward when taking action $a$ from state $s$, and $d_0$ is the initial state distribution. A policy $\pi(a | s)$ defines the probability of taking an action $a$ in state $s$. 

In the following we will indifferently turn to scalar or matrix notation, where $\bm{v}$ denotes a vector, $\bm{M}$ denotes a matrix, and $\bm{v}^T$, $\bm{M}^T$ denote their transpose. A matrix is row (column) stochastic if it has non-negative entries and all of its rows (columns) sum to one. A matrix is \emph{doubly stochastic} if it is both row and column stochastic. We denote with $\mathbb{P}$ the space of doubly stochastic matrices. The $L_\infty$-norm $\Norm{\bm{M}}$ of a matrix is its maximum absolute row sum, while $\Norm[2]{\bm{M}} =  \big( \max \text{eig} \; \bm{M}^T \bm{M} \big)^{\frac{1}{2}}$ and $\Norm[F]{\bm{M}} = \big( \sum_i \sum_j (\bm{M} (i,j))^2 \big)^{\frac{1}{2}}$ are its $L_2$ and Frobenius norms respectively.
We denote with $\bm{1}_n$ a column vector of $n$ ones and with $\bm{1}_{n\times m}$ a matrix of ones with $n$ rows and $m$ columns.
Using matrix notation, $\bm{d}_0$ is a column vector of size $|\mathcal{S}|$ having elements $d_0 (s)$, $\bm{P}$ is a row stochastic matrix of size $(|\mathcal{S}| |\mathcal{A}| \times |\mathcal{S}|)$ that describes the transition model $\bm{P} ((s,a), s') = P(s' | s, a)$, $\bm{\Pi}$ is a row stochastic matrix of size $(|\mathcal{S}| \times |\mathcal{S}| |\mathcal{A}| )$ that contains the policy $\bm{\Pi} (s, (s,a)) = \pi(a | s) $, and $\bm{P}^\pi = \bm{\Pi} \bm{P}$ is a row stochastic matrix of size ($|\mathcal{S}| \times |\mathcal{S}|$) that represents the state transition matrix under policy $\pi$. We denote with $\Pi$ the space of all the stationary Markovian policies.

In the absence of any reward, \ie when $R(s, a) = 0$ for every $(s,a)$, a policy $\pi$ induces, over the MDP $\mathcal{M}$, a Markov Chain (MC)~\cite{levin2017markov} defined by $\mathcal{C} = (\mathcal{S}, P^\pi, d_0)$ where $P^\pi (s' | s) = \bm{P}^\pi (s, s')$ is the state transition model. Having defined the $t$-step transition matrix as $\bm{P}^\pi_t = (\bm{P}^\pi)^t$, the state distribution of the MC at time step $t$ is $\bm{d}_{t}^{\pi} = (\bm{P}^\pi_t)^T \bm{d}_0$, while $\bm{d}^\pi = \lim_{t \rightarrow \infty} \bm{d}^\pi_t$ is the steady state distribution. If the MC is ergodic, \ie aperiodic and recurrent, it admits a unique steady-state distribution, such that $\bm{d}^\pi = (\bm{P}^\pi)^T \bm{d}^\pi$. The mixing time $t_{\text{mix}}$ of the MC describes how fast the state distribution converges to the steady state:
\begin{equation}
	t_{\text{mix}} = \min \big\lbrace t \in \mathbb{N} : \sup\nolimits_{\bm{d}_0} \Norm{\bm{d}^\pi_t - \bm{d}^\pi} \leq \epsilon \big\rbrace,
\end{equation}
where $\epsilon$ is the mixing threshold. 
An MC is reversible if the condition $\bm{P}^\pi  \bm{d}^\pi = (\bm{P}^\pi)^T \bm{d}^\pi$ holds. Let $\bm{\lambda}_\pi $ be the eigenvalues of $\bm{P}^\pi$. For ergodic reversible MCs the largest eigenvalue is 1 with multiplicity 1. Then, we can define the second largest eigenvalue modulus $\lambda_\pi (2)$ and the spectral gap $\gamma_\pi$ as:
\begin{align} \label{def:spectral-gap}
	 &\lambda_\pi (2) = \max_{\lambda_{\pi} (i) \neq 1} |\lambda_\pi (i)|,  & \gamma_\pi = 1 - \lambda_\pi (2).
\end{align}

\section{Optimization Problems for Highly Exploring and Fast Mixing Policies}
\label{sec:objective}
In this section, we define a set of optimization problems whose goal is to identify a stationary Markovian policy that effectively explores the state-action space. The optimization problem is introduced in three steps: first we ask for a policy that maximizes some lower bound to the steady-state distribution entropy, then we foster exploration over the action space by adding a constraint on the minimum action probability, and finally we add another constraint to reduce the mixing time of the Markov chain induced by the policy. 

\subsection{Highly Exploring Policies over the State Space}
\label{sec:highly_exploring_state}
Intuitively, a good exploration policy should guarantee to visit the state space as uniformly as possible.
In this view, a potential objective function is the entropy of the steady-state distribution induced by a policy over the MDP~\cite{hazan2018provably}. The resulting optimal policy is:
\begin{equation} \label{eq:maximum_entropy}
	\pi^* \in \argmax_{\pi \in \Pi} H(\bm{d}^\pi),
\end{equation}
where $H(\bm{d}^\pi) = - \EV_{s \sim d^\pi}  \big[ \log d^\pi (s) \big]$ is the state distribution entropy. Unfortunately, a direct optimization of this objective is particularly arduous since the steady-state distribution entropy is not a concave function of the policy~\cite{hazan2018provably}. To overcome this issue, a possible solution~\cite{hazan2018provably} is to use the conditional gradient method, such that the gradients of the steady-state distribution entropy become the intrinsic reward in a sequence of approximate dynamic programming problems~\cite{bertsekas1995dynamic}.

In this paper, we follow an alternative route that consists in maximizing a lower bound to the policy entropy. In particular, in the following we will consider three lower bounds that lead to as many optimization problems (named Infinity, Frobenius, Column Sum) that show different trade-offs between theoretical guarantees and computational complexity.

\textbf{Infinity}~~From the theory of Markov chains~\cite{levin2017markov}, we know a necessary and sufficient condition for a policy to induce a uniform steady-state distribution (i.e., to achieve the maximum possible entropy).
We report this result in the following theorem.
\begin{restatable}[]{thr}{doublyStochastic} \label{thr:doubly_stochastic}
	Let $\bm{P}$ be the transition matrix of a given MDP. The steady-state distribution $\bm{d}^\pi$ induced by a policy $\pi$ is uniform over $\mathcal{S}$ iff the matrix $\bm{P}^\pi = \bm{\Pi} \bm{P}$ is doubly stochastic.
\end{restatable}
Unfortunately, given the constraints specified by the transition matrix $\bm{P}$, a stationary Markovian policy that induces a doubly stochastic $\bm{P}^\pi$ may not exist.
On the other hand, it is possible to lower bound the entropy of the steady-state distribution induced by policy $\pi$ as a function of the minimum $L_\infty$-norm between $\bm{P}^\pi$ and any doubly stochastic matrix.
\begin{restatable}[]{thr}{entropyBound} \label{thr:entropy_bound}
	Let $\bm{P}$ be the transition matrix of a given MDP and $\mathbb{P}$ the space of doubly stochastic matrices. The entropy of the steady-state distribution $\bm{d}^\pi$ induced by a policy $\pi$ is lower bounded by:
	\begin{equation*}
	 H(\bm{d}^\pi) \geq \log|\mathcal{S}| - |\mathcal{S}|\inf_{\bm{P}^u\in\mathbb{P}}\Norm[\infty]{\bm{P}^u - \bm{\Pi} \bm{P}}^2.
	\end{equation*}
\end{restatable}
The maximization of this lower bound leads to the following constrained optimization problem:
\begin{equation} \label{eq:objective}
	\begin{aligned}
		& \underset{\bm{P}^u\in\mathbb{P}, \bm{\Pi}\in\Pi}{\text{minimize}}
		& & \Norm[\infty]{\bm{P}^u - \bm{\Pi} \bm{P}} 
    \end{aligned}
\end{equation}
It is worth noting that this optimization problem can be reformulated as a linear program with $|\mathcal{S}|^2+|\mathcal{S}||\mathcal{A}|+|\mathcal{S}|$ optimization variables and $2^{|\mathcal{S}|}|\mathcal{S}|+|\mathcal{S}|^2+|\mathcal{S}||\mathcal{A}|$ inequality constraints and $3|\mathcal{S}|$ equality constraints (the linear program formulation can be found in Appendix~\ref{app:lp_infinity}). In order to avoid the exponential growth of the number of constraints as a function of the number of states, we are going to introduce alternative optimization problems.

\textbf{Frobenius}~~It is worth noting that different transition matrices $\bm{P}^\pi$ having equal $\Norm[\infty]{ \bm{P}^u - \bm{P}^\pi }$ might lead to significantly different state distribution entropies $H(\bm{d}^\pi)$, as the $L_\infty$-norm only accounts for the state corresponding to the maximum absolute row sum. The Frobenius norm can better captures the distance between $\bm{P}^u$ and $\bm{P}^\pi$ over all the states, as discussed in Appendix~\ref{apx:three_states}. For this reason, we have derived a lower bound to the policy entropy that replace the $L_\infty$-norm with the Frobenius one.
\begin{restatable}[]{thr}{entropyBoundF} \label{thr:entropy_bound_f}
	Let $\bm{P}$ be the transition matrix of a given MDP and $\mathbb{P}$ the space of doubly stochastic matrices. The entropy of the steady-state distribution $\bm{d}^\pi$ induced by a policy $\pi$ is lower bounded by:
	\begin{equation*}
	 H(\bm{d}^\pi) \geq \log|\mathcal{S}| - |\mathcal{S}|^2\inf_{\bm{P}^u\in\mathbb{P}}\Norm[F]{\bm{P}^u - \bm{\Pi} \bm{P}}^2.
	\end{equation*}
\end{restatable}
It can be shown (see Corollary~\ref{cor:frobenius} in Appendix~\ref{apx:proof}) that the lower bound based on the Frobenius norm cannot be better (i.e., larger) than the one with the Infinite norm. However, we have the advantage that the resulting optimization problem has significantly less constraints than Problem~\eqref{eq:objective}:
\begin{equation} \label{eq:objective_f}
	\begin{aligned}
		& \underset{\bm{P}^u\in\mathbb{P}, \bm{\Pi}\in\Pi}{\text{minimize}}
		& & \Norm[F]{\bm{P}^u - \bm{\Pi} \bm{P}}.
    \end{aligned}
\end{equation}
This problem is a (linearly constrained) quadratic problem with $|\mathcal{S}|^2+|\mathcal{S}||\mathcal{A}|$ optimization variables and $|\mathcal{S}|^2+|\mathcal{S}||\mathcal{A}|$ inequality constraints and $3|\mathcal{S}|$ equality constraints.

\textbf{Column Sum}~~Problems~\eqref{eq:objective} and~\eqref{eq:objective_f} are aiming at finding a policy associated with a state transition matrix that is doubly stochastic. To achieve this result it is enough to guarantee that the column sums of the matrix $\bm{P}^\pi$ are all equal to one~\cite{kirkland2010column}. A measure that can be used to evaluate the distance to a doubly stochastic matrix can be the absolute sum of the difference between one and the column sums: $\sum_{s\in\mathcal{S}}|1-\sum_{s'\in\mathcal{S}}P^\pi(s|s')|=\Norm[1]{\left(\bm{I}-(\bm{\Pi P)^T}\right)\cdot \bm{1}_{|\mathcal{S}|}}$. The following theorem provides a lower bound to the policy entropy as a function of this measure.
\begin{restatable}[]{thr}{entropyBoundCS} \label{thr:entropy_bound_cs}
	Let $\bm{P}$ be the transition matrix of a given MDP. The entropy of the steady-state distribution $\bm{d}^\pi$ induced by a policy $\pi$ is lower bounded by:
	\begin{equation*}
	 H(\bm{d}^\pi) \geq \log|\mathcal{S}| - |\mathcal{S}|\Norm[1]{\left(\bm{I}-(\bm{\Pi P)^T}\right)\cdot \bm{1}_{|\mathcal{S}|}}^2.
	\end{equation*}
\end{restatable}
The optimization of this lower bound leads to the following linear program:
\begin{equation} \label{eq:objective_cs}
	\begin{aligned}
		& \underset{\bm{\Pi}\in\Pi}{\text{minimize}}
		& & \Norm[1]{\left(\bm{I}-(\bm{\Pi P)^T}\right)\cdot \bm{1}_{|\mathcal{S}|}}.
    \end{aligned}
\end{equation}
Besides being a linear program, unlike the other optimization problems presented, Problem~\eqref{eq:objective_cs} does not require to optimize over the space of all the doubly stochastic matrices, thus significantly reducing the number of optimization variables ($|\mathcal{S}|+|\mathcal{S}||\mathcal{A}|$) and constraints ($2|\mathcal{S}|+|\mathcal{S}||\mathcal{A}|$ inequalities and $|\mathcal{S}|$ equalities). The linear program formulation of Problem~\eqref{eq:objective_cs} can be found in Appendix~\ref{app:lp_columnsum}.

\subsection{Highly Exploring Policies over the State and Action Space}
\label{sec:highly_exploring_action}

Although the policy resulting from the optimization of one of the above problems may lead to the most uniform exploration of the state space, the actual goal of the exploration phase is to collect enough information on the environment to optimize, at some point, a goal-conditioned policy~\cite{pong2019skew}.    
To this end, it is essential to have an exploratory policy that adequately covers the action space $\mathcal{A}$ in any visited state. Unfortunately, the optimization of Problems~\eqref{eq:objective}, \eqref{eq:objective_f}, \eqref{eq:objective_cs} does not guarantee even that the obtained policy is stochastic. Thus, we need to embed in the problem a secondary objective that takes into account the exploration over $\mathcal{A}$. This can be done by enforcing a minimal entropy over actions in the policy to be learned, adding to~\eqref{eq:objective}, \eqref{eq:objective_f}, \eqref{eq:objective_cs} the following constraints:
\begin{equation} \label{eq:pi_entropy}
	\pi (a | s) \geq \xi, \quad \forall s \in \mathcal{S}, \quad \forall a \in \mathcal{A},
\end{equation}
where $\xi\in [0, \frac{1}{|\mathcal{A}|}]$.
This secondary objective is actually in competition with the objective of uniform exploration over states. Indeed, an overblown incentive in the exploration over actions may limit the state distribution entropy of the optimal policy. Having a low probability of visiting a state decreases the likelihood of sampling an action from that state, hence, also reducing the exploration over actions. To illustrate that, Figure~\ref{fig:competing_objectives_xi} shows state distribution entropies ($H(\bm{d}^\pi)$) and state-action distribution entropies, \ie $H(\bm{d}^\pi \bm{\Pi})$, achieved by the optimal policy \wrt Problem~\eqref{eq:objective_f} on the Single Chain domain~\cite{furmston2010variational} for different values of $\xi$.
\begin{figure*}[t]
	\begin{minipage}[H]{.38\textwidth}
		\begin{figure}[H]
			\centering
  			\includegraphics[scale=1, valign=b]{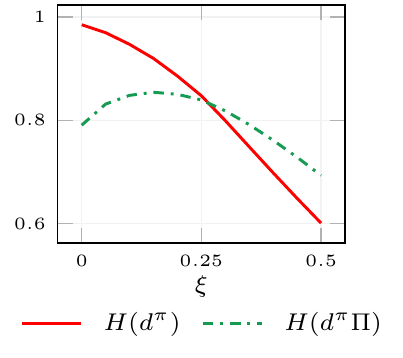}
		\end{figure}
		\caption{State distribution entropy ($H (\bm{d}^\pi)$), state-action distribution entropy ($H (\bm{d}^\pi \bm{\Pi})$) for different values of $\xi$ on the Single Chain domain.}
		\label{fig:competing_objectives_xi}
	\end{minipage}
	\hskip .034\textwidth 
	\begin{minipage}[H]{.58\textwidth}
		\begin{figure}[H]
			\includegraphics[scale=1, valign=t]{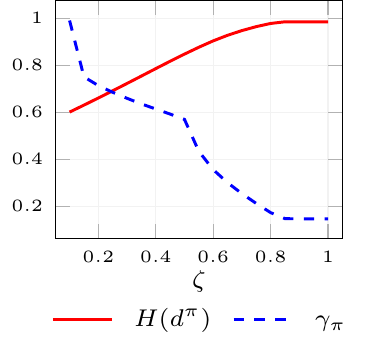}
			\hskip .01\textwidth 
			\includegraphics[scale=0.64, valign=t]{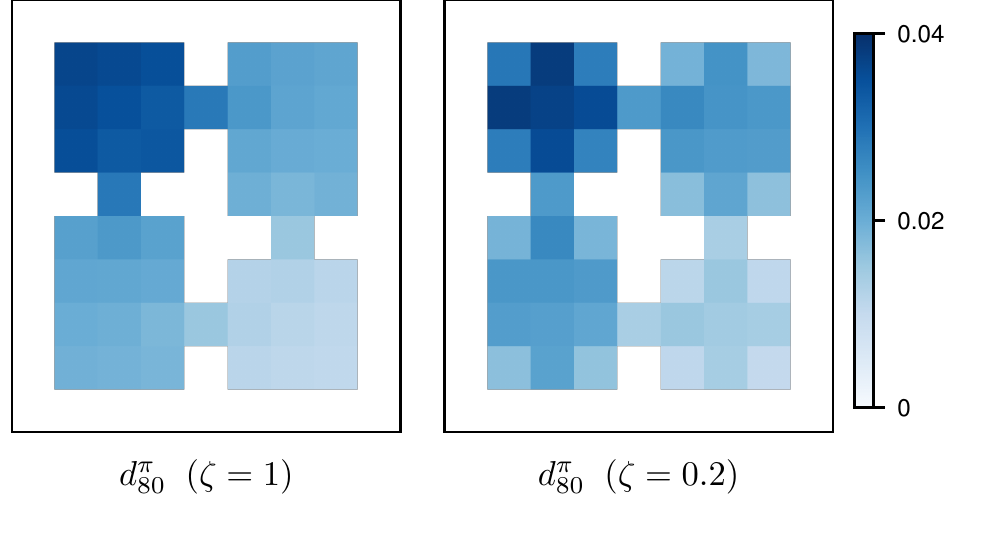}
		\end{figure}
		\caption{State distribution entropy ($H (\bm{d}^\pi)$), spectral gap ($\gamma_\pi$) for different values of $\zeta$ on the Single Chain domain (left). Color-coded state distribution overlaid on a $4$-rooms gridworld for different values of $\zeta$ (right).}
		\label{fig:competing_objectives_zeta}
	\end{minipage}
\end{figure*}

\subsection{An Objective to Make Highly Exploring Policies Mix Faster}
\label{sec:fast_mixing}

In many cases, such as in episodic tasks where the horizon for exploration is capped, we may have interest in trading inferior state entropy for faster convergence of the learned policy. Although the doubly stochastic matrices are equally valid in terms of steady-state distribution, the choice of the target $\bm{P}^u$ strongly affects the mixing properties of the $\bm{P}^\pi$ induced by the policy. Indeed, while an MC with a uniform transition matrix, \ie transition probabilities $\bm{P}^u (s, s') = \frac{1}{|\mathcal{S}|}$ for any $s$, $s'$, mixes in no time, an MC with probability one on the self-loops never converges to a steady state. This is evident considering that the mixing time $t_{\text{mix}}$ of an MC is trapped as follows~\cite[Theorems 12.3 and 12.4]{levin2017markov}:
\begin{equation} \label{eq:mix_trap}
	\frac{1 - \gamma_\pi}{\gamma_\pi} \log  \frac{1}{2 \epsilon} 
	\leq t_{\text{mix}} \leq
	\frac{1}{\gamma_\pi} \log \frac{1}{d^{\pi}_{\text{min}} \epsilon},
\end{equation}
where $\epsilon$ is the mixing threshold, $d^{\pi}_{\text{min}}$ is a minorization of $\bm{d}^\pi$, and $\gamma_\pi$ is the spectral gap of $\bm{P}^\pi$~(\ref{def:spectral-gap}).
From the literature of MCs, we know that a variant of the Problems~\eqref{eq:objective}, \eqref{eq:objective_f} having the uniform transition matrix as target $\bm{P}^u$ and the $L_2$ as matrix norm, is equivalent to the problem of finding the fastest mixing transition matrix $\bm{P}^\pi$~\cite{boyd2004fastest}. However, the choice of this target may overly limit the entropy over the state distribution induced by the optimal policy. Instead, we look for a generalization that allows us to prioritize fast exploration at will. 
Thus, we consider a continuum of relaxations in the fastest mixing objective by embedding in Problems~\eqref{eq:objective} and~\eqref{eq:objective_f} (but not in Problem~\eqref{eq:objective_cs}) the following constraints:
\begin{equation}
	\bm{P}^{u} (s,s') \leq \zeta, \quad \forall s, s' \in \mathcal{S},
\end{equation}
where $\zeta\in [\frac{1}{|\mathcal{S}|}, 1]$. By setting $\zeta=\frac{1}{|\mathcal{S}|}$, we force the optimization problem to consider the uniform transition matrix as a target, thus aiming to reduce the mixing time, while larger values of $\zeta$ relax this objective, allowing us to get a higher steady-state distribution entropy. In Figure~\ref{fig:competing_objectives_zeta} we show how the parameter $\zeta$ affects the trade-off between high steady-state entropy and low mixing times (i.e., high spectral gaps), reporting the values obtained by optimal policies \wrt Problem~\eqref{eq:objective_f} for different $\zeta$.

\section{A Model-Based Algorithm for Highly Exploring and Fast Mixing Policies}
\label{sec:algorithm}

\begin{figure*}[t]
\begin{minipage}[t]{0.47 \textwidth}
	\begin{algorithm}[H]
    \caption{IDE$^{3}$AL}
    \label{alg:ideal}
        \begin{algorithmic}[H]
        	\Statex \textbf{Input}: $\xi$, $\zeta$, batch size $N$
        	\State Initialize $\pi_0$ and transition counts $C \in \mathbb{N}^{|\mathcal{S}|^2 \times |\mathcal{A}|}$
        	\For{ $i = 0, 1, 2, \ldots,$ until convergence }
        		\State Collect N steps with $\pi_i$ and update $C$
        		\State Estimate the transition model as: \\
        						$ \quad \quad \quad \hat{P}_i (s' | s, a) = 
        						\begin{cases} 
        							\frac{C(s' | s, a)}{\sum_{s'} C(s' | s, a)},&\scriptstyle{\text{if} \; C( \cdot | s, a) > 0} \\
        							\scriptstyle{1 / |\mathcal{S}|},&\scriptstyle{\text{otherwise}} \\
        						\end{cases}$
        		\State $\pi_{i+1} \gets$ optimal policy for~\eqref{eq:objective} (or~\eqref{eq:objective_f} or~\eqref{eq:objective_cs}),
        		\State given the parameters $\xi, \zeta,$ and $\hat{P}_i$
        	\EndFor
        	\Statex \textbf{Output}: exploratory policy $\pi_i$
        \end{algorithmic}
	\end{algorithm}
\end{minipage}
\hspace{0.05\linewidth}
\begin{minipage}[t]{0.465 \textwidth}
	\begin{figure}[H]
		\centering
	   		\includegraphics[width=0.59\textwidth]{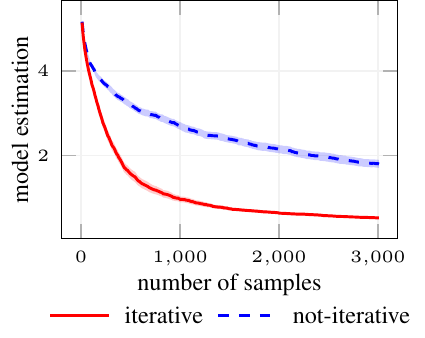}
 			\caption{Model estimation error on the Double Chain with $\xi = 0.1$, $\zeta=0.7$, $N = 10$ (100 runs, 95$\%$ c.i.).}
 			\label{fig:iterative_vs_not}
	\end{figure}
\end{minipage}
\end{figure*}
In this section, we present an approach to incrementally learn a highly exploring and fast mixing policy through interactions with an unknown environment, developing a novel model-based exploration algorithm called \emph{Intrinsically-Driven Effective and Efficient Exploration ALgorithm} (IDE$^{3}$AL). Since Problems~\eqref{eq:objective}, \eqref{eq:objective_f}, \eqref{eq:objective_cs} requires an explicit representation of the matrix $\bm{P}$, we need to estimate the transition model from samples before performing an objective optimization (model-based approach). In tabular settings, this can be easily done by adopting the transition frequency as a proxy for the (unknown) transition probabilities, obtaining an estimated transition model $\hat{P}(s' | s,a)$. However, in hard-exploration tasks, it can be arbitrarily arduous to sample transitions from the most difficult-to-reach states by relying on na\"ive exploration mechanisms, such as a random policy. To address the issue, we lean on an iterative approach in which we alternate model estimation phases with optimization sweeps of the objectives~\eqref{eq:objective}, \eqref{eq:objective_f} or \eqref{eq:objective_cs}. In this way, we combine the benefit of collecting samples with highly exploring policies to better estimate the transition model and the benefit of having a better-estimated model to learn superior exploratory policies. In order to foster the policy towards $(s, a)$ pairs that have never been sampled, we keep their corresponding distribution $\hat{P}(\cdot | s, a)$ to be uniform over all possible states, thus making the pair $(s, a)$ particularly valuable in the perspective of the optimization problem. The algorithm converges whenever the exploratory policy remains unchanged during consecutive optimization sweeps and, if we know the size of the MDP, when all state-action pairs have been sufficiently explored. In Algorithm~\ref{alg:ideal} we report the pseudo-code of IDE$^{3}$AL. Finally, in Figure~\ref{fig:iterative_vs_not} we compare the iterative formulation against a not-iterative one, \ie an approach that collects samples with a random policy and then optimizes the exploration objective off-line. Considering an exploration task on the Double Chain domain~\cite{furmston2010variational}, we show that the iterative form has a clear edge in reducing the model estimation error $\| \bm{P} - \hat{\bm{P}} \|_F$. Both the approaches employ a Frobenius formulation.

\section{Experimental Evaluation}
\label{sec:experiments}

In this section, we provide the experimental evaluation of IDE$^3$AL. First, we show a set of experiments on the illustrative \emph{Single Chain} and \emph{Double Chain} domains~\cite{furmston2010variational,peters2010relative}. The Single Chain consists of $10$ states having $2$ possible actions, one to climb up the chain from state $0$ to $9$, and the other to directly fall to the initial state $0$. The two actions are flipped with a probability $p_\textit{slip} = 0.1$, making the environment stochastic and reducing the probability of visiting the higher states. The Double Chain concatenates two Single Chain into a bigger one sharing the central state $9$, which is the initial state. Thus, the chain can be climbed in two directions. These two domains, albeit rather simple from a dimensionality standpoint, are actually hard to explore uniformly, due to the high shares of actions returning to the initial state and preventing the agent to consistently reach the higher states. Then, we present an experiment on the much more complex \emph{Knight Quest} environment~\cite[Appendix]{fruit2018efficient}, having $|\mathcal{S}| = 360$ and $|\mathcal{A}| = 8$. This domain takes inspiration from classical arcade games, in which a knight has to rescue a princess in the shortest possible time without being killed by the dragon. To accomplish this feat, the knight has to perform an intricate sequence of actions. In the absence of any reward, it is a  fairly challenging environment for exploration. On these domains, we address the task of learning the best exploratory policy in a limited number of samples. Especially, we evaluate these policies in terms of the induced state entropy $H(\bm{d}^\pi)$ and state-action entropy $H(\bm{d}^\pi \bm{\Pi})$.

We compare our approach with \emph{MaxEnt}~\cite{hazan2018provably}, the model-based algorithm to learn maximum entropy exploration that we have previously discussed in the paper, and a count-based approach inspired by the exploration bonuses of MBIE-EB~\cite{strehl2008analysis}, which we refer as \emph{CountBased} in the following. The latter shares the same structure of our algorithm, but replace the policy optimization sweeps with approximate value iterations~\cite{bertsekas1995dynamic}, where the reward for a given state is inversely proportional to the visit count of that state. It is worth noting that the results reported for the MaxEnt algorithm are related to the mixture policy $\pi_\text{mix} = (\mathcal{D}, \alpha)$, where $\mathcal{D} = (\pi_0, \ldots, \pi_{k-1})$ is a set of $k$ $\epsilon$-deterministic policies, and $\alpha \in \Delta_k$ is a probability distribution over $\mathcal{D}$. For the sake of simplicity, we have equipped all the approaches with a little domain knowledge, \ie the cardinality of $\mathcal{S}$ and $\mathcal{A}$. However, this can be avoided without a significant impact on the presented results. For every experiment, we will report the batch-size $N$, and the parameters $\xi$, $\zeta$ of IDE$^3$AL. CountBased and MaxEnt employ $\epsilon$-greedy policies having $\epsilon = \xi$ in all the experiments. In any plot, we will additionally provide the performance of a baseline policy, denoted as \emph{Random}, that randomly selects an action in every state. Detailed information about the presented results, along with an additional experiment, can be found in Appendix~\ref{apx:additional_experiments}.

\begin{figure*}[t]
\begin{minipage}[t]{0.41 \textwidth}
	\begin{table}[H]
		\small
		\setlength\tabcolsep{4pt}
		\begin{center}
		\begin{tabular}{c c c }
			\toprule
			& $H(\bm{d}^\pi)$ &  $\min \bm{d}^\pi$ \\
			\midrule
			\addlinespace[0.1cm]
 			Frobenius ($\xi=0$)				& $0.98$ & $6.4 \cdot 10^{-2}$ \\  
 			\addlinespace[0.1cm]
 			 Infinity ($\xi=0$)					& $0.98$ & $6.4 \cdot 10^{-2}$   \\  
 			\addlinespace[0.1cm]
 			Column Sum ($\xi=0$) 			& $0.98$ & $6 \cdot 10^{-2}$ \\
 			\midrule
 			\addlinespace[0.1cm]
 			Frobenius ($\xi=0.1$)				& $0.94$ & $4.1 \cdot 10^{-2}$ \\  
 			\addlinespace[0.1cm]
 			 Infinity ($\xi=0.1$)					& $0.89$ & $2.6 \cdot 10^{-2}$  \\  
 			\addlinespace[0.1cm]
 			Column Sum ($\xi=0.1$) 			& $0.95$ & $3.8 \cdot 10^{-2}$\\
 			\bottomrule
		\end{tabular}
		\end{center}
	\end{table}
\end{minipage}%
\hspace{0.02\linewidth}
\begin{minipage}[t]{0.55 \textwidth}
	\begin{figure}[H]
		\includegraphics[width=\textwidth]{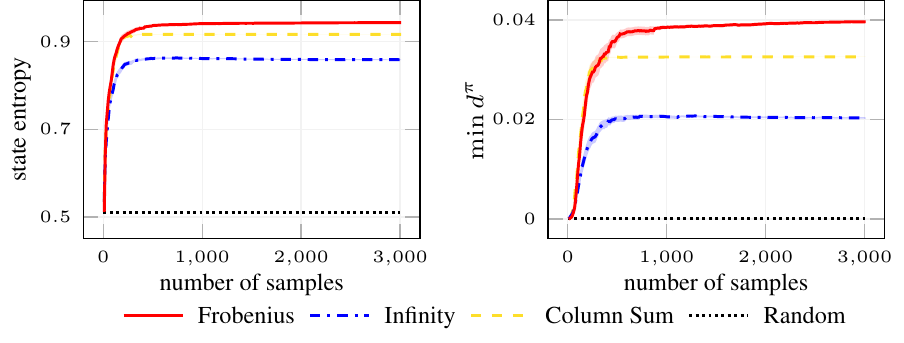}
	\end{figure}
\end{minipage}
\caption{State distribution entropy ($H(\bm{d}^\pi)$) and probability of the least favorable state ($\min \bm{d}^\pi$) for different objective formulations  on the Single Chain domain. We report exact solutions with $\zeta = 0$ (left), and approximate optimizations with $\xi=0.1$, $\zeta=0.7$, $N = 10$ (100 runs, 95$\%$ c.i.) (right).}
\label{fig:chain_comparison}
\end{figure*}
First, in Figure~\ref{fig:chain_comparison}, we compare the Problems~\eqref{eq:objective}, \eqref{eq:objective_f}, \eqref{eq:objective_cs} on the Single Chain environment. On one hand, we show the performance achieved by the exact solutions, \ie computed with a full knowledge of $\bm{P}$. While the plain formulations ($\xi = 0, \zeta = 1$) are remarkably similar, adding a constraint over the action entropy ($\xi = 0.1$) has a significantly different impact. On the other hand, we illustrate the performance of IDE$^3$AL, equipped with the alternative optimization objectives, in learning a good exploratory policy from samples. In this case, the Frobenius clearly achieves a better performance. In the following, we will report the results of IDE$^3$AL considering only the best-performing formulation, which, for all the presented experiments, corresponds to the Frobenius.

In Figure~\ref{fig:double_chain}, we show that IDE$^3$AL compares well against the other approaches in exploring the Double Chain domain. It achieves superior state entropy and state-action entropy, and it converges faster to the optimum. It displays also a higher probability of visiting the least favorable state, and it behaves positively in the estimation of $\hat{\bm{P}}$. 
Notably, the CountBased algorithm fails to reach high exploration due to a detachment problem~\cite{ecoffet2019go}, since it fluctuates between two exploratory policies that are greedy towards the two directions of the chain. By contrast, in a domain having a clear direction for exploration, such as the simpler Single Chain domain,  CountBased ties the explorative performances of IDE$^3$AL (Figure~\ref{fig:single_chain}).
On the other hand, MaxEnt is effective in the exploration performance, but much more slower to converge, both in the Double Chain and the Single Chain. Note that in Figure~\ref{fig:double_chain}, the model estimation error of MaxEnt starts higher than the other, since it employs a different strategy to fill the transition probabilities of never reached states, inspired by~\cite{brafman2002r}. In Figure~\ref{fig:knight_quest}, we present an experiment on the higher-dimensional Knight Quest environment. IDE$^3$AL achieves a remarkable state entropy, while MaxEnt struggles to converge towards a satisfying exploratory policy. CountBased (not reported in Figure~\ref{fig:knight_quest}, see Appendix~\ref{apx:additional_experiments}), fails to explore the environment altogether, oscillating between policies with low entropy.

In Figure~\ref{fig:goal_conditioned}, we illustrate how the exploratory policies learned in the Double Chain environment are effective to ease learning of any possible goal-conditioned policy afterwards. To this end, the exploratory policies, learned by the three approaches through 3000 samples (Figure~\ref{fig:double_chain}), are employed to collect samples in a fixed horizon (within a range from 10 to 100 steps). Then, a goal-conditioned policy is learned off-line through approximate value iteration~\cite{bertsekas1995dynamic} on this small amount of samples. The goal is to optimize a reward function that is 1 for the hardest state to reach (\ie the state that is less frequently visited with a random policy), 0 in all the other states. In this setting, all the methods prove to be rather successful \wrt the baseline, though IDE$^3$AL compares positively against the other strategies.
\begin{figure*}[t]

	\subfloat[][\small Double Chain]{ 
		\includegraphics[]{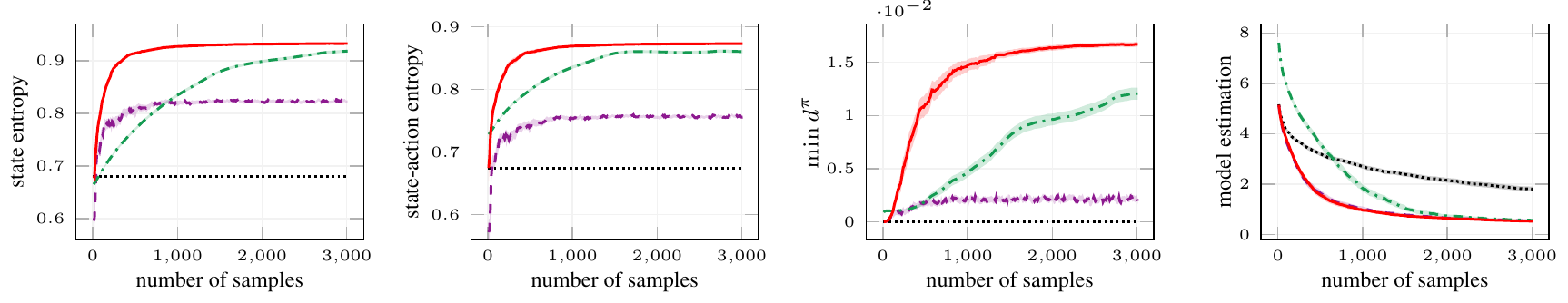}	
		\label{fig:double_chain}
	}%
	
	\subfloat[valign=t][\small Single Chain]{ 
		\includegraphics[]{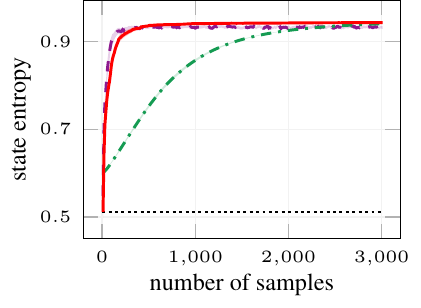}
		\label{fig:single_chain}
	}%
	\subfloat[valign=t][\small Knight Quest]{ 
		\hspace{-0.2cm}
		\includegraphics[]{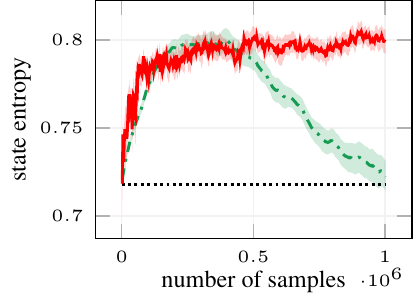}
		\label{fig:knight_quest}
	}%
	\subfloat[valign=t][\small Goal-conditioned]{ 
		\hspace{0.2cm}
		\includegraphics[]{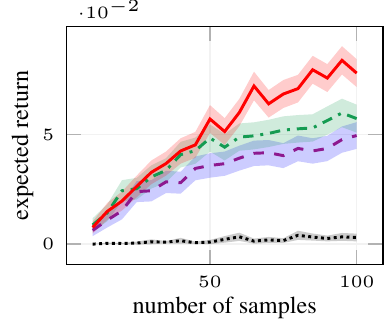}
		\label{fig:goal_conditioned}
	}%
	\subfloat[valign=t][\small Dual comparison]{ 
		\hspace{0.25cm}
		\includegraphics[]{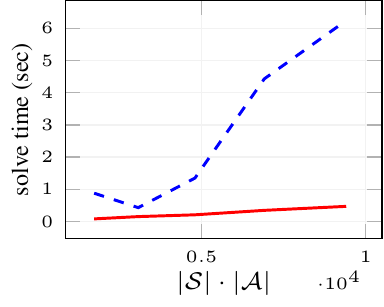} %
		\label{fig:execution_time}
	}%

	\centering \includegraphics[scale=1.1]{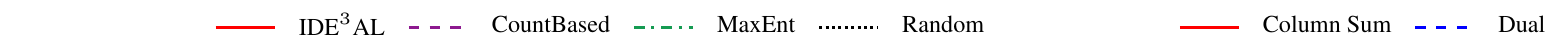}
		
	\caption{Comparison of the algorithms on exploration tasks (a, b, c) and goal-conditioned learning (d), with parameters $\xi=0.1$, $\zeta=0.7$, $N=10$ (a, b, d) and $\xi=0.01$, $\zeta=1$, $N=2500$ (c). (95$\%$ c.i. over 100 runs (a, b), 40 runs (c), 500 runs (d)). Comparison of the solve time (e) achieved by Column Sum and Dual formulations as a function of the number of variables. }
\end{figure*}

\section{Discussion}
\label{sec:discussion}

In this section, we first discuss how the proposed approach might be extended beyond tabular settings and an alternative formulation for the policy entropy optimization. Then, we consider some relevant work related to this paper.

\subsection{Potential Extension to Continuous}
\label{sec:continuous}

We believe that the proposed approach has potential to be extended to more general, continuous, settings, by exploiting the core idea of avoiding a probability concentration on a subset of outgoing transitions from a state. Indeed, a compelling feature of the presented lower bounds is that they characterize an infinite-step property, the entropy of the steady-state distribution, relying only on one-step quantities, i.e., without requiring to unroll several times the state transition matrix $\bm{P}^\pi$ . In addition to this, the lower bounds provide an evaluation for the current policy, and they can be computed for any policy. Thus, we could potentially operate a direct search in the policy space through the gradient of an approximation of these lower bounds. To perform the approximation we could use a kernel for a soft aggregation over regions of the, now continuous, state space.

\subsection{A Dual Formulation}
\label{sec:dual_formulation}

A potential alternative to deal with the optimization of the objective~\eqref{eq:maximum_entropy} is to consider its dual formulation. This is rather similar to the approach proposed in~\cite{tarbouriech2019active} to address the different problem of active exploration in an MDP. The basic idea is to directly maximize the entropy over the \emph{state-action} stationary distribution and then to recover the policy afterwards. In this setting, we define the state-action stationary distribution induced by a policy $\pi$ as $\bm{\omega}^\pi = \bm{d}^\pi \bm{\Pi}$, where $\bm{\omega}^\pi$ is a vector of size $|\mathcal{S}||\mathcal{A}|$ having elements $\omega^\pi (s,a)$. Since not all the distribution over the state-action space can be actually induced by a policy over the MDP, we characterize the set of feasible distributions:
\begin{align*}
	\Omega &= \big\lbrace  \omega \in \Delta( \mathcal{S} \times \mathcal{A} ) : \forall s \in \mathcal{S},  \\   
	  				&\sum_{a \in \mathcal{A}} \omega(s, a) = \sum_{s' \in \mathcal{S}, a' \in \mathcal{A}} P(s | s', a') \omega (s', a') \big\rbrace .
\end{align*}
Then, we can formulate the Dual Problem as:
\begin{equation} \label{eq:dual_objective}
	\begin{aligned}
		& \underset{\bm{\omega} \in \Omega}{\text{maximize}}
		& & H( \bm{\omega} )
    \end{aligned}
\end{equation}
Finally, let $\bm{\omega}^*$ denotes the solution of Problem~\eqref{eq:dual_objective}, we can recover the policy inducing the optimal state-action entropy as $\pi_{\omega^*} (a | s) = \omega^* (s, a) / \sum_{a' \in \mathcal{A}} \omega^* (s, a'), \forall s \in \mathcal{S},  \forall a \in \mathcal{A}$.

The Dual Problem displays some appealing features. Especially, the objective in~\eqref{eq:dual_objective} is already convex, so that it can be optimized right away, and it allows to explicitly maximize the entropy over the state-action space. 
Nonetheless, we think that this alternative formulation has three major shortcomings. First, the optimization of the convex program~\eqref{eq:dual_objective} could be way slower than the optimization of the linear programs Column Sum and Infinity \cite{grotschel1993ellipsoid}. Secondly, it does not allow to control the mixing time of the learned policy, which can be extremely relevant. Lastly, the applicability of the Dual Problem to continuous environments seems far-fetched. It is worth noting that, from an empirical evaluation, the dual formulation does not provide any significant benefit in the entropy of the learned policy \wrt the lower bounds formulations (see Appendix~\ref{apx:additional_experiments}).
Figure~\ref{fig:execution_time} shows how the solve time of the Column Sum scales better with the number of variables ($|\mathcal{S}| |\mathcal{A}|$) in incrementally large Knight Quest domains.

\subsection{Related Work}
\label{sec:related_work}

As discussed in the previous sections, \citeauthor{hazan2018provably}~\shortcite{hazan2018provably} consider an objective not that dissimilar to the one presented in this paper, even if they propose a fairly different solution to the problem. Their method learns a mixture of deterministic policies instead of a single stochastic policy. In a similar flavor, \citeauthor{tarbouriech2019active}~\shortcite{tarbouriech2019active} develop an approach, based on a dual formulation of the objective, to learn a mixture of stochastic policies for active exploration.

Other propose to intrinsically motivate the agent towards learning to reach all possible states in the environment~\cite{lim2012autonomous}. To extend this same idea from the tabular setting to the context of a continuous, high-dimensional state space, \citeauthor{pong2019skew}~\shortcite{pong2019skew} employ a generative model to seek for a maximum-entropy goal distribution. \citeauthor{ecoffet2019go}~\shortcite{ecoffet2019go} propose a method, called Go-Explore, to methodically reach any state by keeping an archive of any visited state and the best trajectory that brought the agent there. At each iteration, the agent draws a promising state from the archive, returns there replicating the stored trajectory (Go), then explores from this state trying to discover new states (Explore).

Another promising intrinsic objective is to make value out of the exploration phase by acquiring a set of reusable skills, typically formulated by means of the option framework~\cite{sutton1999between}. In~\cite{barto2004intrinsically}, a set of options is learned by maximizing an intrinsic reward that is generated at the occurrence of some, user-defined, salient event. The approach proposed by~\citeauthor{bonarini2006incremental}~\shortcite{bonarini2006incremental}, which presents some similarities with the work in~\cite{ecoffet2019go}, is based on learning a set of options to return with high probability to promising states. In their context, a promising state presents high unbalance between the probabilities of the input and output transitions~\cite{bonarini2006self}, so that it is both a hard state to reach, and a doorway to reach many other states. In this way, the learned options heuristically favor an even exploration of the state space.

\section{Conclusions}
\label{sec:conclusion}

In this paper, we proposed a new model-based algorithm, IDE$^3$AL, to learn highly exploring and fast mixing policies. The algorithm outputs a policy that maximizes a lower bound to the entropy of the steady-state distribution. We presented three formulations of the lower bound that differently tradeoff tightness with computational complexity of the optimization. The experimental evaluation showed that IDE$^3$AL is able to achieve superior performance than other approaches striving for uniform exploration of the environment, while it avoids the risk of detachment and derailment~\cite{ecoffet2019go}. Future works could focus on extending the applicability of the presented approach to non-tabular environments, following the blueprint in Section~\ref{sec:continuous}. We believe that this work provides a valuable contribution in view of solving the conundrum on what should a reinforcement learning agent learn in the absence of any reward coming from the environment.

\section*{Acknowledgments}
This work has been partially supported by the Italian MIUR PRIN 2017 Project ALGADIMAR ``Algorithms, Games, and Digital Market''.

\small{
	\bibliographystyle{aaai.bst}
	\bibliography{biblio}
}

\onecolumn
\appendix

\section{Proofs}
\label{apx:proof}

\doublyStochastic*

\begin{proof}
	Let us recall the definition of the steady-state distribution of the MC induced by the policy $\pi$ over the MDP:
	\begin{equation*}
		d^\pi (s) = \sum\nolimits_{s' \in \mathcal{S}} P^\pi (s | s') d^\pi (s'), \quad \forall s \in \mathcal{S}.
	\end{equation*}
	If $\bm{d}^\pi$ is a uniform distribution we have:
	\begin{equation} \label{eq:column-stochastic}
		\sum\nolimits_{s' \in \mathcal{S}} P^\pi (s | s') = 1, \quad \forall s \in \mathcal{S},
	\end{equation}
	then, the state transition matrix $\bm{P}^\pi$ is column stochastic, while it is also row stochastic by definition.
	Conversely, if the matrix $\bm{P}^\pi$ is doubly stochastic, we aim to prove that a $\bm{d}^\pi$ that is not uniform cause an inconsistency in the stationary condition $\bm{d}^\pi = (\bm{P}^\pi)^T \bm{d}^\pi$. Let us consider a perturbation of the uniform $\bm{d}^\pi$, such that $d^\pi (s) = \frac{1}{|\mathcal{S}|}$ for all the states in $\mathcal{S}$ outside of:
	\begin{equation}
		d^\pi (s_h) = \frac{1}{|\mathcal{S}|} + \alpha, \quad d^\pi (s_l) = \frac{1}{|\mathcal{S}|} - \alpha,
	\end{equation}
	where $\alpha$ is a, sufficiently small, positive constant.
	Since $\bm{P}^\pi$ is doubly stochastic, the sum:
	\begin{equation}
		d^\pi (s_h) = \sum\nolimits_{s' \in \mathcal{S}} P^\pi (s_h | s') d^\pi (s'),
	\end{equation}
	is a convex combination of the elements in $\bm{d}^\pi$. Hence, for the stationary condition to hold, we must have $P^\pi(s_h | s_h) = 1$ and $P^\pi(s_h | s) = 0$ for all $s$ different from $s_h$. Nevertheless, a state with probability one on the self-loop cannot have a stationary distribution different from $0$ or $1$.
\end{proof}

%

\entropyBound*

\begin{proof}
 We start with rewriting the entropy of $\bm{d}^\pi$ as follows:
 	\begin{align}\label{eq:entropy}
	 H(\bm{d}^\pi) = -\sum_{s\in \mathcal{S}}d^\pi(s)\log\left( d^\pi(s)\right) = -\sum_{s\in \mathcal{S}}d^\pi(s)\log\left( \frac{d^\pi(s)}{|\mathcal{S}|}|\mathcal{S}|\right) = \log\left(|\mathcal{S}|\right)-D_{KL}(\bm{d}^\pi||\bm{d}^u),
	\end{align}
where $\bm{d}^u$ is the uniform distribution over the state space (all the entries equal to $\frac{1}{|\mathcal{S}|}$) and $D_{KL}(p||q)$ is the Kullback-Leibler (KL) divergence between distribution $p$ and $q$.\\
Using the reverse Pinsker inequality~\cite[p. 1012 and Lemma 6.3]{csiszar2006context}, we can upper bound the KL divergence between $\bm{d}^\pi$ and $\bm{d}^u$:
\begin{align}\label{eq:kl}
 D_{KL}(\bm{d}^\pi||\bm{d}^u) \leq \frac{\Norm[1]{\bm{d}^u - \bm{d}^\pi}^2}{\displaystyle \min_{s\in\mathcal{S}}d^u(s)} = |\mathcal{S}|\cdot\Norm[1]{\bm{d}^u - \bm{d}^\pi}^2.
\end{align}
The total variation between the two steady-state distributions $\bm{d}^\pi$ and $\bm{d}^u$ can in turn be upper bounded by (see~\cite{schweitzer1968perturbation}):
\begin{align}\label{eq:perturbation}
 \Norm[1]{\bm{d}^u - \bm{d}^\pi} \leq \Norm[\infty]{\bm{Z}}\Norm[\infty]{\bm{P}^u-\bm{\Pi P}},
\end{align}
where $\bm{Z}=\left(\bm{I}-\bm{P^u}+\bm{1}_{|\mathcal{S}|}\frac{1}{|\mathcal{S}|}\right)^{-1}$ is the fundamental matrix and $\bm{P}^u$ is any doubly-stochastic matrix ($\bm{P}^u\in\mathbb{P}$). Since the fundamental matrix associated to any doubly-stochastic matrix is row stochastic~\cite{hunter2010some}, then $\Norm[\infty]{\bm{Z}}=1$. Furthermore, since the bound in Equation~\eqref{eq:perturbation} holds for any $\bm{P}^u\in\mathbb{P}$, we can rewrite the bound as follows:
\begin{align}\label{eq:perturbation_ds}
 \Norm[1]{\bm{d}^u - \bm{d}^\pi} \leq \inf_{\bm{P}^u\in\mathbb{P}}\Norm[\infty]{\bm{P}^u-\bm{\Pi P}}.
\end{align}
Combining Equations~\eqref{eq:kl} and~\eqref{eq:perturbation_ds} we get an upper bound to the KL divergence, which, once replaced in Equation~\eqref{eq:entropy}, provides the lower bound in the statement and concludes the proof.

\end{proof}

\entropyBoundF*

\begin{proof}
 From the properties of the matrix norms~\cite{petersen2008matrix}, we have that for any $n\times n$ matrix $\bm{M}$ it holds:
$$  \Norm[F]{\bm{M}} \leq \frac{1}{\sqrt{n}}\Norm[\infty]{\bm{M}}.$$
As a consequence:
$$\inf_{\bm{P}^u\in\mathbb{P}} \Norm[F]{\bm{P}^u-\bm{\Pi P}}^2 \geq \frac{1}{|\mathcal{S}|}\Norm[\infty]{\overline{\bm{P}^u}-\bm{\Pi P}}^2 \geq  \frac{1}{|\mathcal{S}|} \inf_{\bm{P}^u\in\mathbb{P}}\Norm[\infty]{\bm{P}^u-\bm{\Pi P}}^2,$$
where $\overline{\bm{P}^u}=\arg\inf_{\bm{P}^u\in\mathbb{P}}\Norm[F]{\bm{P}^u-\bm{\Pi P}}^2$.
Combining this inequality with the result in Theorem~\ref{thr:entropy_bound} concludes the proof.
\end{proof}

\entropyBoundCS*

\begin{proof}
 We start with defining the vector $\bm{c}$ that results from the difference between the vector of ones $\bm{1}_{|\mathcal{S}|}$ and the vector of the column sums: $\bm{c} = \left(\bm{I}-(\bm{\Pi P)^T}\right)\cdot \bm{1}_{|\mathcal{S}|}$. We denote with $\widehat{\bm{P}_x}$ the matrix obtained from $\bm{P}^\pi$ by adding $\bm{c}^T$ to the row corresponding to state $x$:
 
 $$\widehat{\bm{P}_x}(s,s') = \left\{\begin{aligned}
                                    &\bm{P}^\pi(s,s')+c(s'),\quad&\text{if $s=x$}\\
                                    &\bm{P}^\pi(s,s'), \quad&\text{otherwise}
                                   \end{aligned}\right.$$
 It is worth noting that, since $\sum_{s\in\mathcal{S}}c(s)=0$, the column sums and the row sums of matrix $\widehat{\bm{P}_x}$ are all equal to $1$. Nonetheless, $\widehat{\bm{P}_x}$ is not guaranteed to be doubly stochastic since its entries can be lower than $0$. However, it is possible to show that 
 $$\inf_{\bm{P}^u\in\mathbb{P}}\Norm[\infty]{\bm{P}^u-\bm{P}^\pi}\leq \Norm[\infty]{\widehat{\bm{P}_x}-\bm{P}^\pi} = \Norm[1]{\bm{c}}.$$
 When $\widehat{\bm{P}_x}$ is doubly stochastic, the above inequality holds by definition. When $\widehat{\bm{P}_x}$ has negative entries, it is always possible to transform it to a doubly stochastic matrix without increasing the $L_\infty$ distance from $\bm{P}^\pi$.
 In order to remove the negative entries of $\widehat{\bm{P}_x}$, we need to trade probability with the other states, so as to preserve the row sum. Each state that gives probability to state $x$, will receive the same amount of probability taken by the columns corresponding to positive values of the vector $\bm{c}$. In order to illustrate this procedure, we consider a four-state MDP and a policy $\pi$ that leads to the following state transition matrix:
 $$\bm{P}^\pi=\begin{bmatrix}
                     0.8 & 0.2 & 0 & 0 \\
                     0 & 0.9 & 0.1 & 0\\
                     0.3 & 0.5 & 0.1 & 0.1\\
                     0.8 & 0.1 & 0.1 & 0
                    \end{bmatrix}.$$
 The corresponding vector $\bm{c}$ is
 $$\bm{c}=\begin{bmatrix}
           -0.9 & -0.7 & 0.7 & 0.9
          \end{bmatrix}^T.$$
Summing $\bm{c}^T$ to the first row of $\bm{P}^\pi$ we get:
 $$\widehat{\bm{P}_{s_1}}=\begin{bmatrix}
                     -0.1 & -0.5 & 0.7 & 0.9 \\
                     0 & 0.9 & 0.1 & 0\\
                     0.3 & 0.5 & 0.1 & 0.1\\
                     0.8 & 0.1 & 0.1 & 0
                    \end{bmatrix}.$$
Since we have two negative elements, to get a doubly stochastic matrix we can modify the matrix as follows:
\begin{itemize}
 \item move $0.1$ from element $(3,1)$ to $(1,1)$ and (to keep the row sum equal to 1) move $0.1$ from $(1,3)$ to $(3,3)$
 \item move $0.5$ from element $(2,2)$ to $(1,2)$ and (to keep the row sum equal to 1) move $0.5$ from $(1,3)$ to $(2,3)$
\end{itemize}
The resulting matrix is:
$$\widehat{\bm{P}}=\begin{bmatrix}
                     0 & 0 & 0.1 & 0.9 \\
                     0 & 0.4 & 0.6 & 0\\
                     0.2 & 0.5 & 0.2 & 0.1\\
                     0.8 & 0.1 & 0.1 & 0
                    \end{bmatrix}\in \mathbb{P}.$$
 The described procedure yields a doubly stochastic matrix $\widehat{\bm{P}}$ such that $\Norm[\infty]{\widehat{\bm{P}}-\bm{P}^\pi}\leq\Norm[1]{\bm{c}}$.
 Combining this upper bound with the result in Theorem~\ref{thr:entropy_bound} concludes the proof.
\end{proof}

\begin{coroll}\label{cor:frobenius}
 The bound in Theorem~\ref{thr:entropy_bound} is never less than the bound in Theorem~\ref{thr:entropy_bound_f}.
\end{coroll}

\begin{proof}
 From the properties of the matrix norms~\cite{petersen2008matrix}, we have that for any $n\times n$ matrix $\bm{M}$ it holds:
 \begin{equation}
  \frac{\Norm[\infty]{\bm{M}}}{\sqrt{n}} \leq \Norm[F]{\bm{M}} \leq \sqrt{n}\Norm[\infty]{\bm{M}}.
 \end{equation}
As a consequence:
$$|\mathcal{S}|^2\inf_{\bm{P}^u\in\mathbb{P}}\Norm[F]{\bm{P}^u-\bm{\Pi P}}^2 \geq |\mathcal{S}|\Norm[\infty]{\overline{\bm{P}^u}-\bm{\Pi P}}^2 \geq |\mathcal{S}|\inf_{\bm{P}^u\in\mathbb{P}}\Norm[\infty]{\bm{P}^u-\bm{\Pi P}}^2,$$
where $\overline{\bm{P}^u}=\arg\inf_{\bm{P}^u\in\mathbb{P}}\Norm[F]{\bm{P}^u-\bm{\Pi P}}$.
It follows that 
$$\log|\mathcal{S}| -  |\mathcal{S}|^2\inf_{\bm{P}^u\in\mathbb{P}}\Norm[F]{\bm{P}^u-\bm{\Pi P}}^2 \leq \log|\mathcal{S}| - |\mathcal{S}|\inf_{\bm{P}^u\in\mathbb{P}}\Norm[\infty]{\bm{P}^u-\bm{\Pi P}}^2.$$
\end{proof}

%
%

\section{Optimization Problems}\label{app:Optimization}
\subsection{Linear program formulation of Problem~\eqref{eq:objective}}\label{app:lp_infinity}
Problem~\eqref{eq:objective} can be rewritten as follows:
\begin{equation}  \label{eq:lp_infinity}
	\begin{aligned}
		& \underset{\bm{P}^u, \bm{\Pi},v}{\text{minimize}}
		& & v \\
		& \text{subject to}
		& & \sum_{s'\in\mathcal{S}}|P^{u} (s'|s)-P^{\pi} (s'|s)| \leq v, \quad \forall s \in \mathcal{S}, \\
		& & & \bm{\Pi} (s, (s,a)) \geq 0, \quad \forall s \in \mathcal{S}, \quad \forall a \in \mathcal{A}, \\
		& & & P^u (s'|s) \geq 0, \quad \forall s \in \mathcal{S}, \quad \forall s' \in \mathcal{S}, \\
		& & & \sum_{a\in\mathcal{A}}\bm{\Pi} (s, (s,a)) = 1, \quad \forall s \in \mathcal{S}, \\
		& & & \sum_{s'\in\mathcal{S}}P^{u} (s'|s) = 1, \quad \forall s \in \mathcal{S}, \\
		& & & \sum_{s'\in\mathcal{S}}P^{u} (s|s') = 1, \quad \forall s \in \mathcal{S}.
	\end{aligned}
\end{equation}
The first set of inequality constraints can be transformed in a set of linear inequality constraints. Each constraint is obtained by removing the absoulte values and considering a different permutation of the signs in front of the terms in the summation. As a result, if the original summation contains $n$ elements, the number of linear constraints is $2^n$. Since this process needs to be done for each state $s\in\mathcal{S}$, the first set of constraints can be replaced by $|\mathcal{S}|2^{|\mathcal{S}|}$.

\subsection{Linear program formulation of Problem~\eqref{eq:objective_cs}}\label{app:lp_columnsum}

Let $\bm{v}$ be a vector of length $|\mathcal{S}|$.
Problem~\eqref{eq:objective} can be rewritten as follows:
\begin{equation}
	\begin{aligned}
		& \underset{\bm{\Pi},\bm{v}}{\text{minimize}}
		& & \sum_{s\in\mathcal{S}}v(s) \\
		& \text{subject to}
		& & 1-\sum_{s'\in\mathcal{S}}P^{\pi} (s|s') \leq v(s), \quad \forall s \in \mathcal{S}, \\
		& & &\sum_{s'\in\mathcal{S}}P^{\pi} (s|s')-1 \leq v(s), \quad \forall s \in \mathcal{S}, \\
		& & & \bm{\Pi} (s, (s,a)) \geq 0, \quad \forall s \in \mathcal{S}, \quad \forall a \in \mathcal{A}, \\
		& & & \sum_{a\in\mathcal{A}}\bm{\Pi} (s, (s,a)) = 1, \quad \forall s \in \mathcal{S}.
	\end{aligned}
\end{equation}

\section{Illustrative Example}
\label{apx:three_states}

The example in Figure~\ref{fig:three_states} shows that the Frobenius norm can better capture the distance between a transition matrix $\bm{P}$ and a doubly stochastic $\bm{P}^u \in \mathbb{P}$ \wrt the $L_\infty$-norm. Indeed, the $L_\infty$-norm only accounts for the state which corresponds to the maximum absolute row sum of the difference $\bm{P}^u - \bm{P}$, while the Frobenius norm considers the difference across all the states. In the example, we see two transition matrices $\bm{P}_1$ and $\bm{P}_2$ that are equally bad in the worst state ($s_0$), thus, have equal $L_\infty$-norm. However, $\bm{P}_1$ is fairly unbalanced also in the other states, where $\bm{P}_2$ is uniform instead, and so it is clearly preferable in view of the uniform exploration objective.

\begin{figure*}[t]
\begin{minipage}[b]{0.6 \textwidth}
	\begin{figure}[H]
	\centering
		\includegraphics[scale=1, valign=t]{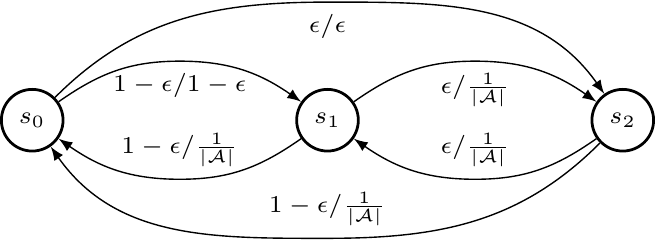}
	\end{figure}
\end{minipage}
\begin{minipage}[b]{0.3 \textwidth}
	\begin{table}[H]
		\footnotesize
		\begin{center}
		\begin{tabular}{c c c }
			\toprule
			& $\bm{P}_1$ &  $\bm{P}_2$  \\
			\midrule
			\addlinespace[0.1cm]
 			$L_\infty$ 				& $0.60$ & $0.60$ \\  
 			\addlinespace[0.1cm]
 			 Frobenius 				& $0.73$ & $0.42$   \\  
 			\addlinespace[0.1cm]
 			$H(\bm{d}^\pi)$ 	& $0.93$ & $0.99$\\
 			\bottomrule
		\end{tabular}
		\end{center}
	\end{table}
\end{minipage}
\caption{Graphical representation of a Markov chain (left), having on the edges the transition probabilities $\bm{P}_1 (s_i, s_j) / \bm{P}_2 (s_i, s_j)$. On the right, a table providing the values of $L_\infty$-norm, and Frobenius norm of the difference \wrt a uniform $\bm{P}^u$, along with state distribution entropies.}
\label{fig:three_states}
\end{figure*}

\section{Experimental Evaluation: Further Details}
\label{apx:additional_experiments}

In the following, we provide further details on the experimental evaluation covered by Section~\ref{sec:experiments}. First, for the sake of clarity, we report the pseudo-code of MaxEnt and CountBased algorithms, which we have compared with our approach. Then, for any presented experiment, we recap the full set of parameters employed, we show the value of the exact solutions, and a characterization of the solve time, for all the different formulations. Finally, we illustrate an additional experiment in the River Swim domain~\cite{strehl2008analysis}.

As a side note, it is worth reporting that our implementation of the optimization Problems~\eqref{eq:objective}, \eqref{eq:objective_f}, \eqref{eq:objective_cs} is based on the \emph{CVXPY} framework~\cite{cvxpy} and makes use of the \emph{MOSEK} optimizer.

\subsection{Algorithms: Pseudo-Code}

In Algorithm~\ref{alg:maxent}, we report the pseudo-code of the MaxEnt algorithm~\cite{hazan2018provably}. In Algorithm~\ref{alg:countb} the pseudo-code of the CountBased algorithm, which is inspired by the exploration bonus of MBIE-EB~\cite{strehl2008analysis}.
\begin{algorithm}[H]
    \caption{MaxEnt}
    \label{alg:maxent}
        \begin{algorithmic}[H]
        	\Statex \textbf{Input}: $\epsilon$, batch size $N$, step size $\eta$
        	\State Initialize $\pi_0$, transition counts $T \in \mathbb{N}^{|\mathcal{S}|^2 \times |\mathcal{A}|}$, state visitation counts $C \in \mathbb{N}^{|\mathcal{S}|}$
        	\State Set $\alpha_0 = 1$ and $\mathcal{D}_0 = \lbrace  \pi_0  \rbrace$, let $\pi_{\text{mix}, 0} = (\mathcal{D}_0, \alpha_0)$
        	\For{ $i = 0, 1, 2, \ldots,$ until convergence }
        		\State Sample $\pi \sim \pi_{\text{mix}, i}$, collect $N$ steps with $\pi$, and update $T, C$
        		\State Estimate the transition model as: \\
        						$ \quad \quad \quad \hat{P}_i (s' | s, a) = 
        						\begin{cases} 
        							\frac{T(s' | s, a)}{\sum_{s'} T(s' | s, a)},& \text{if} \; T( \cdot | s, a) > 0 \\
        							1,  & \text{if} \; T( \cdot | s, a) = 0 \; \text{and} \; s' = s \\
        							0,  \; & \text{otherwise} \\
        						\end{cases}$
        		\State Define the reward function as: \\
        						$ \quad \quad \quad R_i (s) = 
        						\begin{cases} 
        							\big( \nabla H(\bm{d}^{\pi_{\text{mix}, i}}) \big)_s,&\text{if} \; C(s) > 0 \\
        							\; \log |\mathcal{S}|, & \text{otherwise} \\
        						\end{cases}$
        		\State Compute $\pi_{i+1}$ as the $\epsilon$-greedy policy for the MDP $\mathcal{M} = (\mathcal{S}, \mathcal{A}, \hat{P}, R_i, d_0)$
        		\State Set $\alpha_{i + 1} = ((1 - \eta) \alpha_i, \eta), \mathcal{D}_{i + 1} = (\pi_0, \dots, \pi_i, \pi_{i + 1})$, update $\pi_{\text{mix}, i + 1} = (\mathcal{D}_{i + 1}, \alpha_{i + 1})$
        	\EndFor
        	\Statex \textbf{Output}: exploratory policy $\pi_{\text{mix}, i + 1}$
        \end{algorithmic}
\end{algorithm}

\begin{algorithm}[H]
    \caption{CountBased}
    \label{alg:countb}
        \begin{algorithmic}[H]
        	\Statex \textbf{Input}: $\epsilon$, batch size $N$
        	\State Initialize $\pi_0$, transition counts $C \in \mathbb{N}^{|\mathcal{S}|^2 \times |\mathcal{A}|}$, state visitation counts $N \in \mathbb{N}^{|\mathcal{S}|}$
        	\For{ $i = 0, 1, 2, \ldots,$ until convergence }
        		\State Collect N steps with $\pi_i$ and update $C, N$
        		\State Estimate the transition model as: \\
        						$ \quad \quad \quad \hat{P}_i (s' | s, a) = 
        						\begin{cases} 
        							\frac{C(s' | s, a)}{\sum_{s'} C(s' | s, a)},&\scriptstyle{\text{if} \; C( \cdot | s, a) > 0} \\
        							\scriptstyle{1 / |\mathcal{S}|},&\scriptstyle{\text{otherwise}} \\
        						\end{cases}$
        		\State Define the reward function as $R_i (s) = \frac{1}{N(s) + 1} $
        		\State Compute $\pi_{i+1}$ as the $\epsilon$-greedy policy for the MDP $\mathcal{M} = (\mathcal{S}, \mathcal{A}, \hat{P}, R_i, d_0)$
        	\EndFor
        	\Statex \textbf{Output}: exploratory policy $\pi_i$
        \end{algorithmic}
\end{algorithm}

\subsection{Experiments}

For any experiment covered by Section~\ref{sec:experiments}, we provide (Table~\ref{tab:exp_details}) the cardinality of the state space $|\mathcal{S}|$, the cardinality of the action space $|\mathcal{A}|$, the value of the parameters $\xi$, $\zeta$ of IDE$^3$AL, the parameter $\epsilon$ of MaxEnt and CountBased, the number of iterations $I$, and the batch-size $N$, which are shared by all the approaches. For every domain, we report (Table~\ref{tab:optimal_values}) the value of the exact solution of Problems~\eqref{eq:objective}, \eqref{eq:objective_f}, \eqref{eq:objective_cs}, \eqref{eq:dual_objective}. In Table~\ref{tab:computation_times}, we provide the time to find a solution for the Problems~\eqref{eq:objective}, \eqref{eq:objective_f}, \eqref{eq:objective_cs} that we experienced running the optimization on a single-core general-purpose CPU. For any experiment (except Knight Quest), we show additional figures reporting the performance of all the formulations of IDE$^3$AL.

The River Swim environment~\cite{strehl2008analysis} mimic the task of crossing a river either swimming upstream or downstream. Thus, the action of swimming upstream fails with high probability, while the action of swimming downstream is deterministic. Due to this imbalance in the effort needed to cross the environment in the two directions, it is a fairly hard task in view of uniform exploration.

\begin{table}[H]
		\small
		\setlength\tabcolsep{12pt}
		\begin{center}
		\begin{tabular}{c c c c c c c c}
			\toprule
			& $|\mathcal{S}|$ & $|\mathcal{A}|$ & $\xi$ & $\zeta$ & $\epsilon$ & $I$ & $N$ \\
			\midrule
			Single Chain & $10$ & $2$ & $0.1$ & $0.7$ & $0.1$ & $300$ & $10$ \\
			Double Chain & $20$ & $2$ & $0.1$ & $0.7$ & $0.1$ & $300$ & $10$ \\
			Knight Quest & $360$ & $8$ & $0.01$ & $1$ & $0.01$ & $400$ & $2500$ \\
			River Swim & $6$ & $2$ & $0.1$ & $0.7$ & $0.1$ & $300$ & $10$ \\
 			\bottomrule
		\end{tabular}
		\caption{Full set of parameters employed in the presented experiments.}
		\label{tab:exp_details}
		\end{center}
\end{table}

\begin{table}[H]
		\small
		\setlength\tabcolsep{8pt}
		\begin{center}
		\begin{tabular}{c c c c}
			\toprule
			& \multicolumn{3}{c}{$H(\bm{d}^\pi)$}\\
			\midrule
			& Single Chain & Double Chain & Knight Quest \\
			\midrule
			Infinity 				& 	$0.983$	& $0.967$	& $0.720$ \\
			Frobenius 			& 	$0.984$	& 	$0.970$	& $0.850$ \\
			Column Sum 		& $0.985$	& 	$0.961$	& $0.837$ \\
			Dual 					& $0.985$	& 	$0.971$	& $0.859$ \\
 			\bottomrule
		\end{tabular}
		\caption{State distribution entropy related to the exact solution for all the problem formulations.}
		\label{tab:optimal_values}
		\end{center}
\end{table}

\begin{table}[H]
		\small
		\setlength\tabcolsep{8pt}
		\begin{center}
		\begin{tabular}{c c c c}
			\toprule
			& \multicolumn{3}{c}{Solve Time (sec)}\\
			\midrule
			& Single Chain & Double Chain & Knight Quest \\
			\midrule
			Infinity 				& $\simeq 10^{-3}$ 	&  $\simeq 10^{-2}$	& $3.23$ \\
			Frobenius 			& $\simeq 10^{-3}$ 	& 	$\simeq 10^{-2}$	& $3.07$ \\
			Column Sum 		& $\simeq 10^{-5}$ 	& 	$\simeq 10^{-4}$	& $0.03$ \\
 			\bottomrule
		\end{tabular}
		\caption{Time to solve the optimization for all the problem formulations (reported in seconds).}
		\label{tab:computation_times}
		\end{center}
\end{table}

\subsection*{Single Chain}

\begin{figure}[H]
	\centering
	   	\includegraphics{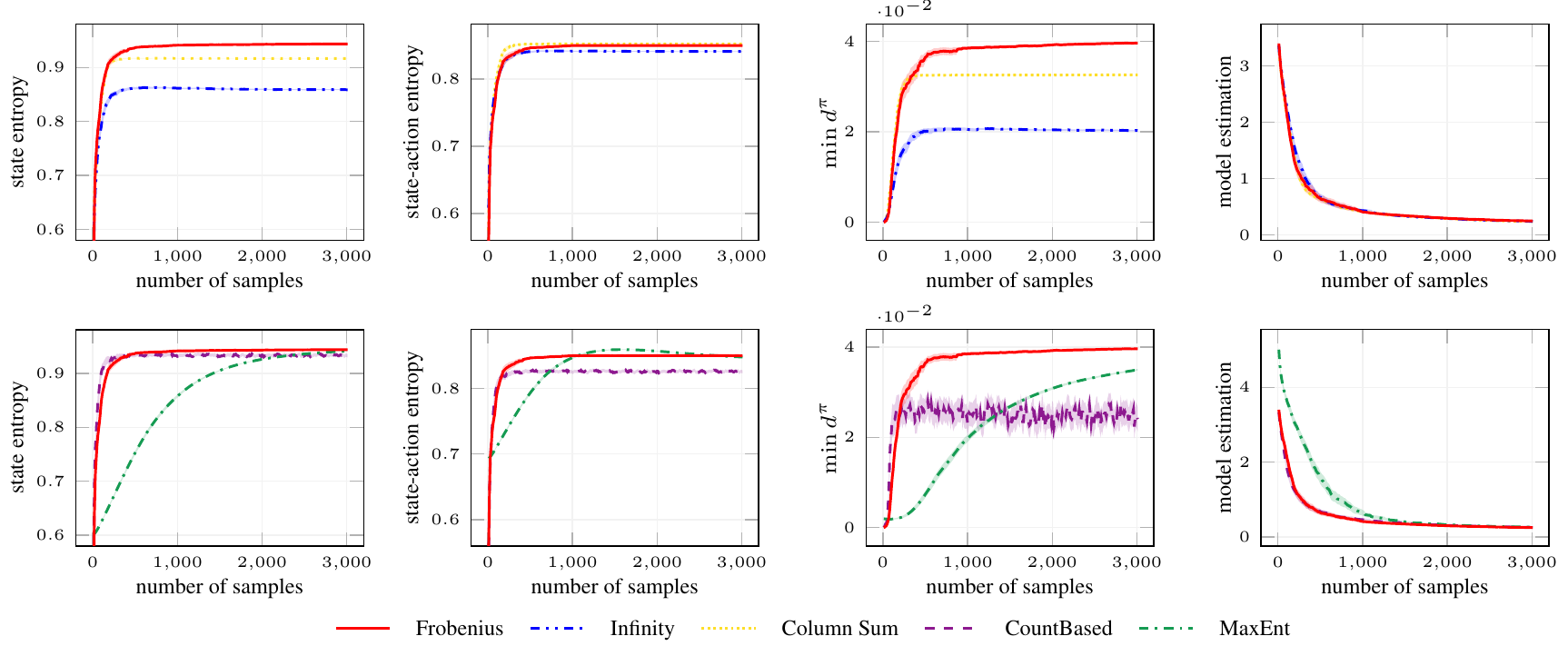}
 		\caption{Comparison of the algorithms' exploration performances in the Single Chain environment with parameters $\xi=0.1$, $\zeta=0.7$, $N=10$ (100 runs, 95$\%$ c.i.).}
 		\label{fig:single_chain_complete}
\end{figure}

\subsection*{Double Chain}

\begin{figure}[H]
	\centering
	   	\includegraphics{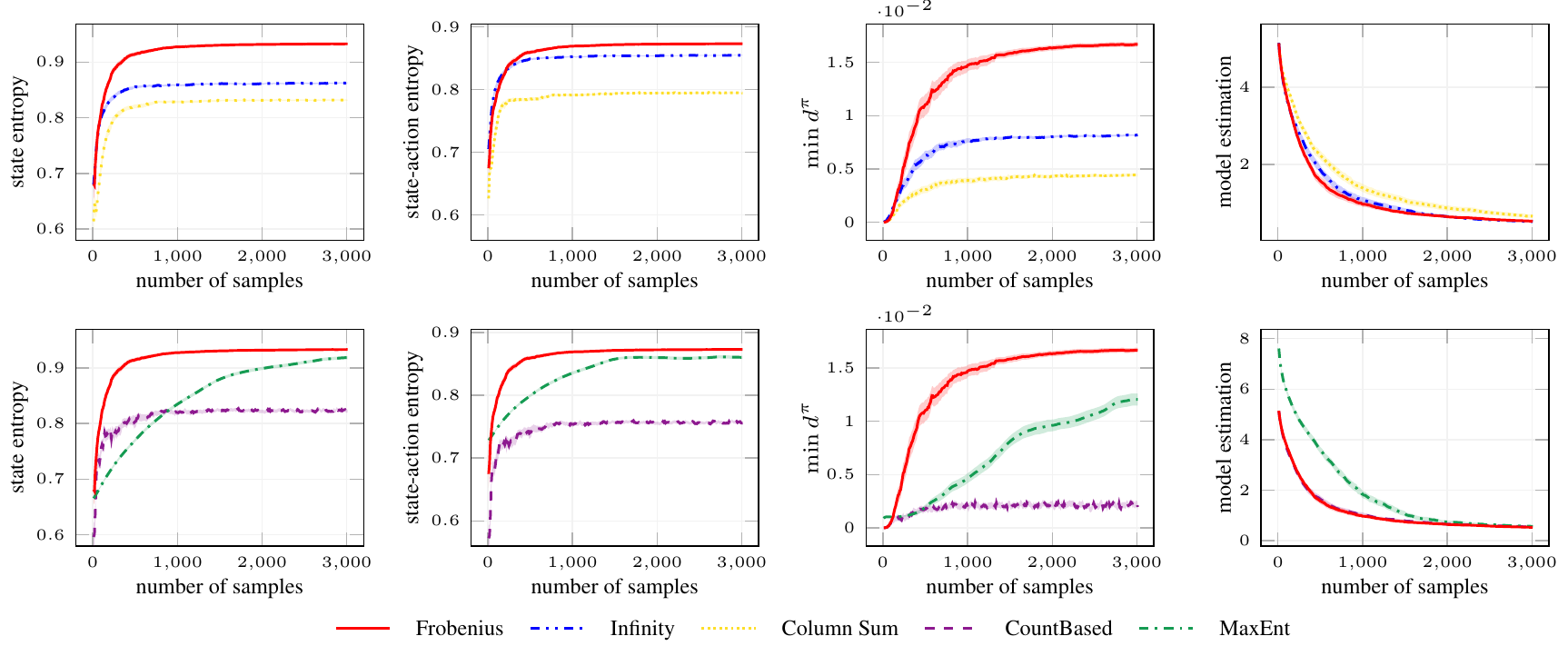}
 		\caption{Comparison of the algorithms' exploration performances in the Single Chain environment with parameters $\xi=0.1$, $\zeta=0.7$, $N=10$ (100 runs, 95$\%$ c.i.).}
 		\label{fig:double_chain_complete}
\end{figure}

\subsection*{Knight Quest}

\begin{figure}[H]
	\centering
	   	\includegraphics{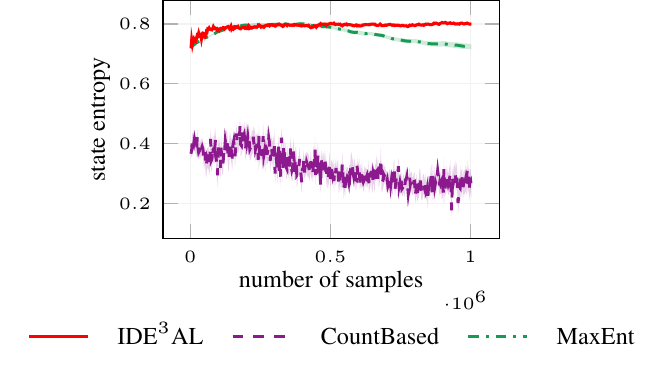}
 		\caption{Comparison of the algorithms' state entropy in the Knight Quest environment with parameters $\xi=0.01$, $\zeta=1$, $N=2500$ (40 runs, 95$\%$ c.i.).}
 		\label{fig:knight_quest_complete}
\end{figure}

\subsection*{Goal-Conditioned}

\begin{figure}[H]
	\centering
	   	\includegraphics{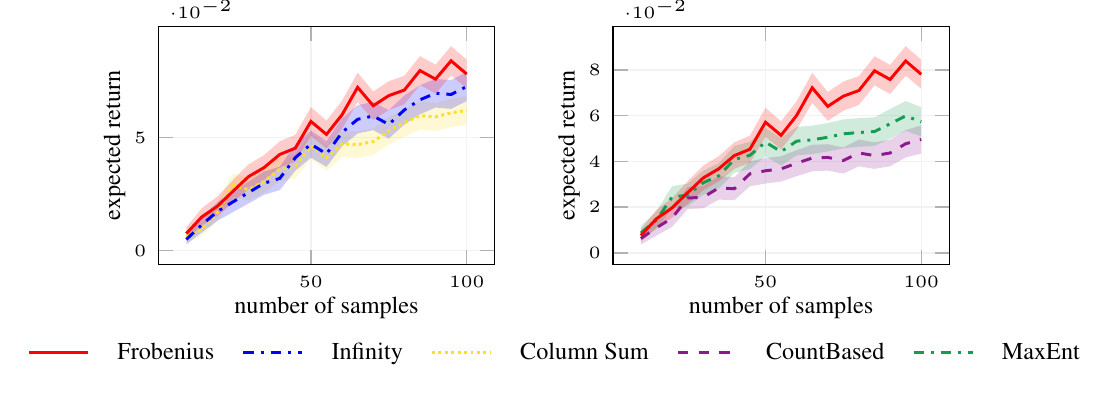}
 		\caption{Comparison of the algorithms on a goal-conditioned learning task in the Double Chain environment with parameters $\xi=0.1$, $\zeta=0.7$, $N=10$ (500 runs, 95$\%$ c.i.).}
 		\label{fig:goal_conditioned_complete}
\end{figure}

\subsection*{River Swim}

\begin{figure}[H]
	\centering
	   	\includegraphics{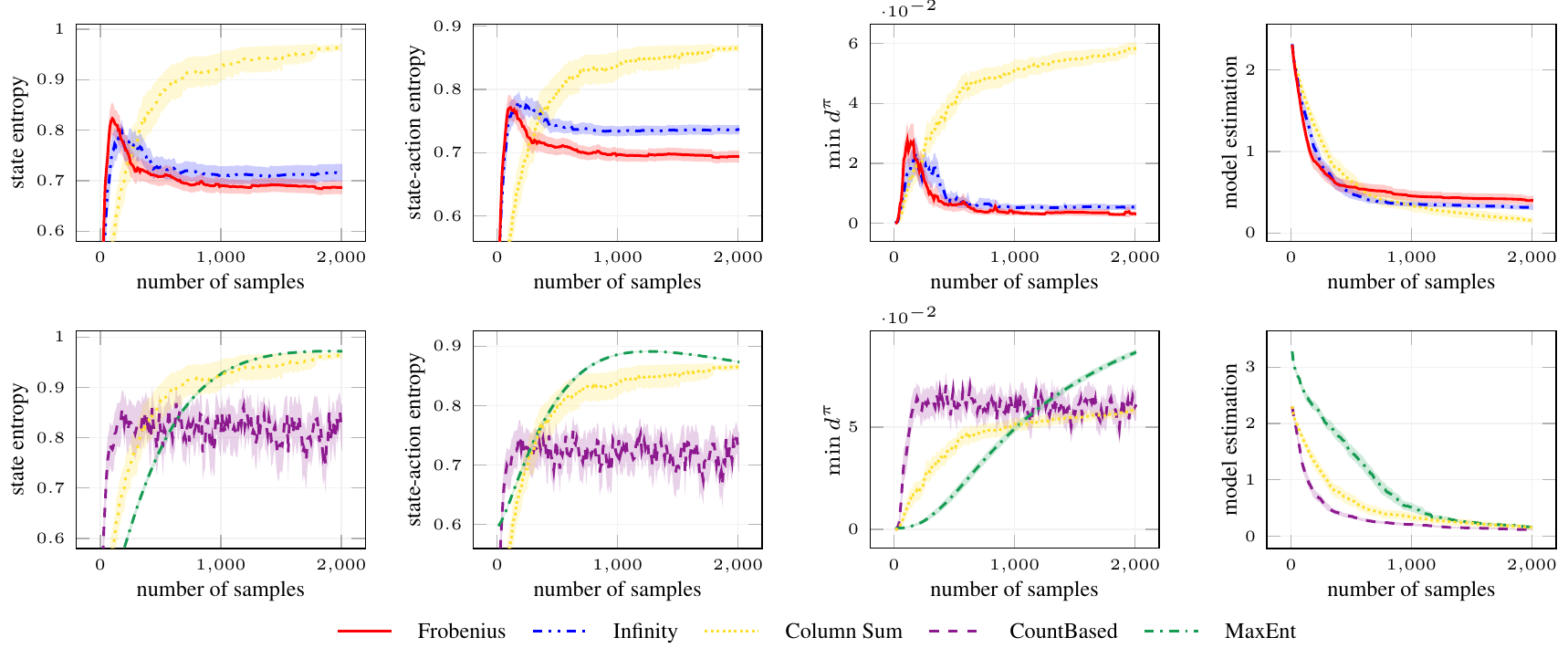}
 		\caption{Comparison of the algorithms' exploration performances in the River Swim environment with parameters $\xi=0.1$, $\zeta=0.7$, $N=10$ (100 runs, 95$\%$ c.i.).}
 		\label{fig:river_complete}
\end{figure}

\end{document}